\documentclass[sigconf]{acmart}
\AtBeginDocument{%
  }

\setcopyright{acmlicensed}
\copyrightyear{2025}
\acmYear{2025}
\acmConference[ICAIL 2025]{20th International Conference on Artificial Intelligence and Law}{June 16--20, 2025}{Chicago, IL}

\usepackage{tikz}
\usetikzlibrary{shapes.geometric, arrows, positioning}
\tikzstyle{box} = [rectangle, rounded corners, minimum width=1.6cm, minimum height=0.8cm,text centered, draw=white, fill=white]
\tikzstyle{arrow} = [thick,->,>=stealth]

\tikzstyle{vertex}=[circle, draw, inner sep=0pt, minimum size=14pt]
\newcommand{\vertex}{\node[vertex]}

\newtheorem{remark}{Remark}%

\newcommand{\prem}{\text{\tt Prem}} 
\newcommand{\conc}{\text{\tt Conc}} 
\newcommand{\sub}{\text{\tt Sub}} 
 
\newcommand{\sr}{\text{\tt StRules}} 
 
\newcommand{\rules}{\text{\tt Rules}}
\newcommand{\norms}{\text{\tt Norms}}

\begin{document}

\title{Cross-Border Legal Adaptation of Autonomous Vehicle Design based on Logic and Non-monotonic Reasoning}


\author{Zhe Yu}
\authornote{Both authors contributed equally to this research.}\authornote{Corresponding author.}
\orcid{0000-0002-8763-4473}
\affiliation{%
  \institution{Institute of Logic and Cognition, Department of Philosophy, \\Sun Yat-sen University}
  \city{Guangzhou}
  \country{China}
}
\email{yuzh28@mail.sysu.edu.cn}

\author{Yiwei Lu}
\authornotemark[1]
\author{Burkhard Schafer}
\email{Y.Lu-104@sms.ed.ac.uk}
\email{B.Schafer@ed.ac.uk}
\affiliation{%
  \institution{School of Law, Old College, University of Edinburgh}
  \city{Edinburgh}
  \country{UK}}

\author{Zhe Lin}
\affiliation{%
  \institution{Department of Philosophy, \\Xiamen University}
  \city{Xiamen}
  \country{China}
}
\email{pennyshaq@163.com}



\begin{abstract}
This paper focuses on the legal compliance challenges of autonomous vehicles in a transnational context. We choose the perspective of designers and try to provide supporting legal reasoning in the design process. Based on argumentation theory, we introduce a logic to represent the basic properties of argument-based practical (normative) reasoning, combined with partial order sets of natural numbers to express priority. Finally, through case analysis of legal texts, we show how the reasoning system we provide can help designers to adapt their design solutions more flexibly in the cross-border application of autonomous vehicles and to more easily understand the legal implications of their decisions. 
\end{abstract}




\maketitle

\section{Introduction}\label{sec:intro}

A widely recognized requirement for autonomous vehicles is their capability to operate across jurisdictions (e.g., cross-border freight transport and interregional passenger services) while maintaining consistent compliance with diverse regulatory frameworks \cite{Pattinson2020}. 
In the traditional automotive industry, designers are only responsible for meeting the production requirements of the vehicle, while driving is the responsibility of the driver. Therefore, the traditional model is for manufacturers to design the physical properties of the respective prototypes according to the requirements of different regions. However, in the context of automated vehicles, part of the driving responsibility is also borne by the designers, so the designers need to face more complex legal reasoning during the design process. Combined with the application goal of cross-border driving, the traditional design model is no longer applicable in terms of both economy and feasibility. Therefore, we believe that intelligent tools that help designers understand the legal implications of their design decisions and efficiently make adjustments to their design plans would be beneficial. 

Design-centric regulatory approaches are gaining traction in the field of artificial intelligence. 
In terms of the technology of autonomous vehicles, engineers and designers need to do the following, based on existing legal documents such as the EU AI Act~\cite{ACT} and the proposal by the Law Societies of England and Wales: provide engineering solutions that demonstrate the product's compliance with legal requirements; and document and demonstrate that design choices aim for legal compliance. This places previously unanticipated requirements and costs on designers and manufacturers for the legal compliance of products, especially in more complex legal scenarios. Traditional legal reasoning systems focus on two ideas: creating an electronic judge~\cite{mills2016artificial} or a perfect law-abiding citizen's brain from machine learning methods~\cite{licari2024italian}. In the context of this paper, we argue that addressing these challenges from the designer's perspective during the design process can yield valuable insights. 

To capture the non-monotonic nature of legal reasoning and its unique explanatory requirements, in previous studies \cite{lujurix22,lu2023legal}, we have investigated the application of argumentation-based reasoning systems to support designers in identifying and addressing conflicts between autonomous vehicles and specific legal systems. This includes the development of explanation mechanisms that align more closely with legal interpretive principles, as well as a preliminary exploration of design adaptation between two distinct legal frameworks in transnational contexts. 
In the current paper, we introduce a logic to represent the basic properties of argument-based practical (normative) reasoning, and combine a partial order set of natural numbers ($\mathbb{N}$) to express priorities. The resulting logic is the negative and predicate extension of a fragment of the Lambek Calculus \cite {Lambek1958} ($\mathbf{L}$), called LN, as a basic system for structured argumentation theory. 

Compared with the system presented in our previous work~\cite{lujurix22,lu2023legal}, as well as with other systems offering similar functionalities, the advantages of our proposed design are as follows: 
First, it can ensure the rationality of legal reasoning, i.e., it will not produce results that violate the intuition of legal rationality. Second, the design of LN focuses on the understandability of the reasoning process and results, which is conducive to our goal of using it by designers with non-legal backgrounds. They can easily understand and use the reasoning results without the need for systematic legal knowledge or an in-depth understanding of the logical provisions of the system. Besides, key features of LN include: 
\begin{enumerate}
    \item Quantified labels for propositional preferences and preference-based reasoning, enabled by its derivation-tracing capability. 
    \item A minimalistic foundation comprising only widely accepted principles of practical reasoning based on argumentation. 
\end{enumerate}

While first-order logic proves undecidable and overly complex for our purposes, LN constitutes a conservative extension of the Lambek calculus. As demonstrated in prior research \cite{Lambek1958, Docherty2018}, when constrained to specific formula classes relevant to practical applications, this logic can be effectively reduced to a non-associative Lambek calculus with assumptions. Crucially, this variant has been proven to be polynomial-time decidable, and its computational implementation can leverage the well-established CYK (Cocke-Kasami-Younger) algorithm - a mature parsing technique with extensive applications in natural language processing. 

The following sections are organized as follows: In Section \ref{sec:rw}, we discuss related works; in Section \ref{sec:formal}, we provide definitions of the non-monotonic reasoning system (an argumentation theory) for cross-country autonomous driving and its underlying logic; in Section \ref{sec:application}, we demonstrate the application of the system combined with legal analysis; finally, in Section \ref{sec:conc}, we present the conclusions.

\section{Related Work}\label{sec:rw}
Eastman and Collins et al. \cite{eastman2023comparative} emphasized the critical role legal frameworks play as gatekeepers influencing AV operation, development, and liability for global deployment of AI vehicles. Dhabu and Ankita \cite{dhabu2024legal} extend this by highlighting complexities in cross-border AI applications, advocating for flexible compliance strategies due to the uncertainty and time-consuming nature of developing unified global standards \cite{cihon2019standards}. Aligning with this flexible compliance approach, Kingsdon \cite{kingston2017using} developed a practical intelligent assistance system to meet GDPR requirements, an idea in much way in line with this study.

Addressing conflicting information in cross-border legal scenarios has become essential, requiring manufacturers to balance legal requirements, societal expectations, and economic benefits \cite{singer2024corporate}. Fakeyede et al. \cite{fakeyede2023navigating} demonstrate the necessity of smoothly transitioning between distinct privacy frameworks such as GDPR and CCPA to manage legal risks. Eggers et al. \cite{eggers2022drivers} further emphasize user preferences and brand experiences as influential factors in AV adoption decisions, underscoring the need for customizable and market-sensitive strategies.

Our system builds on structured argumentation frameworks, incorporating rule-based knowledge representation and preference handling, akin to \textit{ASPIC}$^+$ \cite{MP13}. Priority orderings have traditionally been used to handle legal hierarchies, but this approach is not directly visible and difficult to trace, making it less accessible to non-specialists. LN's preference method offers a clearer comprehension and more efficient reasoning traceability, as detailed later. Our objective is a logically simple yet robust system tailored specifically for autonomous driving applications in cross-border contexts, facilitating easy comprehension and effective preference reasoning for general users. 

\section{Formal Logic and Definition of Arguments}\label{sec:formal}

\begin{definition}
The formal language $\mathcal{L}$ is defined inductively as follows ($t$ denotes a term and $\alpha$ denotes a formula):
\small\[
t ::= x
\mid c
\mid f(t_1, \ldots, t_n)
\]
where $x$ and $c$ represent variables and constants respectively, $f$ represents a function. Let $A$ represent a predicate and \( p \) an atomic formula ($At (t_1, \ldots, t_n)$). 
\small\[A ::= p\in At
\mid \neg A 
\mid A_1 \cdot A_2 
\mid \forall x A  
\mid \bot
\]
Define $\exists A= \neg \forall \neg A$ and $u=\neg \bot$.
\end{definition}
\begin{definition}\footnote{Note that a sequence is a multiset where the order of elements is essential. For instance, A, B is different from B, A, while A, A, B is different from A, B.}
The set of sequences denoted by $\Gamma$ of formulas is defined inductively as follows:
\small\[\Gamma ::=A
\mid \Gamma, \Gamma  
\]
\end{definition}

We define $\mathcal{L}$- sequent. A sequent is
of the form \small\[
\Gamma\Rightarrow A
\]
where the antecedent is a sequence of $\mathcal{L}$ formulas and succedent is a formula in $\mathcal{L}$. The axiomatization of logic is presented by sequent way.  Hereafter, we denote this logic by LN. 

\begin{definition}
LN consists of the following axioms and rules.
\begin{small}
\begin{itemize}
	\item[$(1)$] Axiom Schemes: 
        \[ A\Rightarrow A \quad(\mathrm{Id})\]
	\item[$(2)$] Rules of Inference:
	\begin{itemize}
       \item[] 
       \[\frac{\Gamma\Rightarrow A}{\Gamma, u\Rightarrow A}\quad({\mathrm{uL})}\quad\frac{}{\Rightarrow u}\quad({\mathrm{uR})}\footnote{
       Rules for the unit $u$ are to indicate that the unit of the formula is the empty string.}
       \quad\frac{\Delta \Rightarrow B\quad\Gamma, B\Rightarrow A}{\Gamma,\Delta\Rightarrow A}\quad(\mathrm{Cut})
       \]\[\frac{\Gamma, A, B \Rightarrow C}{\Gamma_1, A\cdot B,\Rightarrow C}\quad(\cdot\mathrm{L})\quad \frac{\Gamma\Rightarrow A\quad \Delta\Rightarrow B}{\Gamma, \Delta\Rightarrow A\cdot B}\quad(\cdot\mathrm{R})
       \]
       \[\frac{A, \Gamma \Rightarrow }{\Gamma\Rightarrow \neg A}{(\neg\mathrm{L})}\,\frac{\Gamma\Rightarrow A}{\Gamma, \neg A\Rightarrow}{(\neg\mathrm{R})}\, \frac{\Gamma_1, A\Rightarrow B}{\Gamma_1,\neg\neg A\Rightarrow B}\,(\neg \mathrm{L})\,\frac{\Gamma\Rightarrow A}{\Gamma_\Rightarrow \neg\neg A}\,(\neg \mathrm{R})\]
       \[\frac{\Gamma, A[t/x]\Rightarrow B}{\Gamma, \forall xA[x] \Rightarrow B} \quad(\forall\mathrm{L}) \quad \frac{\Gamma\Rightarrow A[y/x]}{\Gamma\Rightarrow \forall x A[x]}\,\quad(\forall\mathrm{R})\]
    
       \[\frac{\Gamma,(\Delta_1,\Delta_2),\Delta_3\Rightarrow A}{\Gamma,\Delta_1,(\Delta_2,\Delta_3)\Rightarrow A} \quad (\mathrm{Ass})\quad \frac{\Gamma, A\Rightarrow B}{A,\Gamma\Rightarrow B}{(\mathrm{Cyc})}\]
        \end{itemize}
	
\end{itemize}
\end{small}
where in ($\forall\mathrm{R}$) $y$ may not occur in $\Gamma, \Delta, A[-]$.
\end{definition}

 By $\vdash_{\mathrm{LN}} \Gamma\Rightarrow A$ where $\Gamma=A_1,\ldots,A_n$, one denotes that the sequent $\Gamma\Rightarrow A$ is valid (provable) in LN. 
 \begin{proposition}
If $\vdash_{\mathrm{LN}}\Gamma_1, A, \Gamma_2\Rightarrow B$, then $\vdash_{\mathrm{LN}}\Gamma_1, \neg B, \Gamma_2\Rightarrow \neg A$.  
 \end{proposition}

 \begin{remark}
 We consider sequences instead of sets in argumentation reasoning. For example, consider the inference rule $A_1,A_2\ldots A_n\Rightarrow B$. 
In application contexts, the arrangement of arguments is naturally essential, which implies that the ordering of $A_1, A_2, \ldots, A_n$ is also essential. 
 Moreover, the arguments may sometimes be resource sensitive. For example, consider two different arguments $Ag_1$ and $Ag_2$ with the same conclusion, which is "Mary has five dollars" (denoted by $A$). Then one concludes that "Mary has ten dollars" but not "Mary has 5 dollars". Hence in such reasoning, $A$ and $A$ cannot be considered equal to $A$. 
 For the above reason, we consider some basic substructural logics as our logic base in the present paper.
 \end{remark}

 \begin{remark}\label{Decidable:LN}
The logic LN is the predicate extension of the $``\cdot,\neg"$-fragment of Lambek Calculus with unit and cyclic involutive negation in \cite{buszkowski2019involutive} when $\neg A=A\backslash 0$ ($0$ is denoted by $\bot$ here) in \cite{buszkowski2019involutive}. The propositional fragment of LN is decidable, which means that for any given sequent $\Gamma\Rightarrow A$, there exists a terminable algorithm to tell you whether $\Gamma \Rightarrow A$ is provable in it.
 \end{remark}
 
In the present paper, we always assume that $\mathbb{N}_0\subseteq At$. In the following, we use $\mathbb{N}$ to denote $\mathbb{N}_0$.
\begin{definition}
Let $\{0\}\subseteq \mathbb{\overline{N}}\subseteq\mathbb{N}$ be finite and $\mathcal{\overline{N}}$ be the set of sequents generated from $(\mathbb{N},\leq)$ by the following way
\[\mathcal{\overline{N}}=\{i\Rightarrow j| i\leq j\in (\mathbb{N},\leq) \,\& \,i,j\in \mathbb{\overline{N}}\}. \]    
\end{definition}

\begin{definition}\label{Def:Prinformula}
A prin formula is a formula of the following form:
\[B=A\cdot i_1\cdot i_2\cdot\ldots\cdot i_n\& i_j (1\leq j\leq n)\in \mathbb{N}\& A\not\in \mathbb{N}\]
We call the natural number following after subformula $A$ by first prin of $A$, second prin of $A$ and so on.\footnote{We assume that the natural numbers correspond to the immediate front nonnumeric subformula. For example, consider the formula $p\cdot 1\cdot 5 \cdot q$. Then $1,5$ correspond to $p$.}  If $A\not\in\mathbb{N}$ is an atomic formula, then we say that $B$ is literal. Let us denote the set of all prin formulas by $\mathcal{PF}$.
\end{definition}
\begin{definition}

A set of legal norms $\mathcal{N}$ is a finite set of sequents in the following form 
\[ l_1\cdot\ldots\cdot l_n\Rightarrow l \] 
where $l_i,l$ are literals $ 1\leq i\leq n$.
\end{definition}

By $\mathcal{\overline{N}}\vdash_{\mathrm{LN}} \Gamma\Rightarrow A$, we mean that $\Gamma \Rightarrow A$ is derivable from $\mathcal{N}$ in LN. 
Define $\mathcal{R}_{s}=\{\Gamma\Rightarrow A| \mathcal{\overline{N}}\vdash_{\mathrm{LN}}\Gamma\Rightarrow A\}$. A sequent $s$ is called a strict inference rule if $s\in \mathcal{R}_{s}$ and denoted by $\Gamma\Rightarrow_s A$. 
Let $\mathcal{N}\cap \mathcal{\overline{N}}=\emptyset$ and $\mathcal{N}\cap \mathcal{R}_{s}=\emptyset$. In the following, we always assume that $\mathcal{N}$ is finite.  A seqent $s\in \mathcal{N}$ is denoted by $\Gamma\Rightarrow_n A$. Let $\mathcal{L}_{s}$ be the set of all sequents. Obviously $\mathcal{\overline{N}}\subseteq\mathcal{R}_{s}\subseteq \mathcal{L}_{s}$ and  $\mathcal{N}\subseteq \mathcal{L}_{s}$. In the following, we write $\vdash_{\mathrm{LN}}$ to denote $\mathcal{\overline{N}}\vdash_{\mathrm{LN}}$.

\begin{definition}
We define some functions on $\mathcal{L}$: 
\begin{itemize}
    \item en($A$)=$A'$ such that $A'$ is obtained from $A$ by eliminating all  formulas in $\mathbb{N}$.
    \item nn($A$)=$A'$ such that $A'$ is obtained from $A$ b by eliminating all formulas in $\mathcal{L}/\mathbb{N}$. nn($A$)=$0$ if the process returns an empty string.\footnote{For instance, let $A=p\cdot 3\cdot 5$, then nn($A$)=$3\cdot 5$. Let $A=p\cdot q$, then nn($A$)=$0$.}
    \item nx($A$)$=\{ p| p\in \mathbb{N} \& p \, \mathrm{appears\, in\, x \,position\, of}\, A\}$.\footnote{For instances, let $A=p\cdot 3\cdot 5$, then n2($A$)=\{5\} while n1($A$)=\{3\}.}
\end{itemize}
\end{definition}

 An argumentation theory for legal reasoning (LeAr) based on LN and $\mathcal{L}$ can be defined as follows. 
  \begin{definition}[LeAr]\label{def:lesac}
    A $LeAr$ is a pair $( \mathcal{N}, \mathcal{K})$, 
    where:
    \begin{itemize}
        \item $\mathcal{N}$ is a finite set of legal norms; 
        \item $\mathcal{K}\subseteq \mathcal{L}$ is a finite set of accepted (justified) knowledge.
    \end{itemize}
    
\end{definition}  
\begin{definition}
Based on \(\mathcal{K}\), arguments can be constructed via $\mathrm{LN}$. Let the following functions be defined for any argument \(Ag\):
\begin{itemize}
    \item \(\texttt{Prem}(Ag)\): returns the set of all formulas from \(\mathcal{K}\) that were used to build \(Ag\).
 \item \(\texttt{Conc}(Ag)\): returns the conclusion of \(Ag\).
 \item \(\texttt{Sub}(Ag)\): returns the set of all subarguments of \(Ag\).
  \item \(\rules(Ag)\): returns the set of all rules applied in \(Ag\).
 \item \(\texttt{Norms}(Ag)\): returns the set of all norms applied in \(Ag\).
 \item \(\sr(Ag)\): returns the set of all strict rules (rules in \(\mathcal{R}_s\)) applied in \(Ag\). 
 \item \(\texttt{TConc}(Ag)\): \texttt{TConc}($Ag$)=\texttt{enConc}($Ag$)
\end{itemize}
\end{definition}

\begin{definition}[Arguments]\label{def:argument}
	Let $\mathcal{A}$ be the set of all constructible arguments based on $\mathrm{LN}$ and a LeAr.
    Each argument $Ag\in\mathcal{A}$ is defined as follows: 
	\begin{enumerate}
	\item $Ag=A \in \mathcal{K}$. In this case $\prem(Ag)=\{A\}$, $\conc(Ag)=A$, $\norms(Ag)=\emptyset$,  $\sr(Ag)=\emptyset$, $\rules(Ag)=\emptyset$, and $\sub(Ag)=\{A\}$;
		\item $Ag$  is $Ag_{1}$  $\ldots$, $Ag_{n}$ $\mapsto Ag$, where $ Ag_{1} $, $ \ldots $, $Ag_{n}\in\mathcal{A}$ satisfying that  $ \conc(Ag_{1}) $, $ \ldots $, $ \conc(Ag_{n}) $ $\Rightarrow_s A$. In this case 
		$ \prem(Ag)=\prem(Ag_{1})\cup \ldots \cup \prem(Ag_{n}) $,  $\conc(Ag)=A$, 
		$ \norms(Ag)=\norms$ $(Ag_{1})\cup \ldots \cup \norms(Ag_{n}) $,  $\sr(Ag)=\sr(Ag_1)\cup\ldots\cup\sr(Ag_n)\cup$ $\{\conc(Ag_{1}) $, $ \ldots $ , $\,$ $\conc(Ag_{n}) $ $\Rightarrow_s A\}$, $\rules(Ag)$ $=\norms(Ag)\cup\sr(Ag)$, and
		$ \sub(Ag)=\sub(Ag_{1})\cup \ldots \cup \sub(Ag_{n})\cup\{Ag\}$;  
	\item $Ag$  is $Ag_{1}$  $\ldots$, $Ag_{n}$ $\mapsto Ag$, and there exists a defeasible rule (norm) in $\mathcal{N}$ such that  $\conc(Ag_{1}) $, $ \ldots $, $ \conc(Ag_{n}) $ $\Rightarrow_{n} A$. In this case $ \prem(Ag)$ $=\prem(Ag_{1})\cup \ldots \cup \prem(Ag_{n}) $, $\conc(Ag)=A$, 
		$ \norms(Ag)\\=\norms(Ag_{1})$ $\cup \ldots \cup \norms(Ag_{n}) \cup \{ $ $\conc(Ag_{1}) $, $ \ldots $ , $\,$ $\conc(Ag_{n}) $ $\Rightarrow_n A\}$, $\sr(Ag)=\sr(Ag_1)\cup\ldots\cup\sr(Ag_n)$, $\rules(Ag)=\norms(Ag)\cup\sr(Ag)$, and
		$ \sub(Ag)=\sub(Ag_{1})$ $\cup \ldots \cup \sub(Ag_{n})\cup\{Ag\}$.

	\end{enumerate}
\end{definition}
\begin{remark}
 Let $\mathcal{M} = \mathcal{N} \cup \overline{\mathcal{N}}$. We denote by $\mathrm{LN}(\mathcal{M})$ the logic $\mathrm{LN}$ enriched with the set of sequents $\mathcal{M}$ as assumptions; that is, all sequents in $\mathcal{M}$ are treated as additional axioms in $\mathrm{LN}$. $\mathcal{M} \vdash_{\mathrm{LN}} \Gamma \Rightarrow A$ indicates that $\Gamma \Rightarrow A$ is derivable from $\mathcal{M}$ in $\mathrm{LN}$—equivalently, that $\Gamma \Rightarrow A$ is provable in $\mathrm{LN}(\mathcal{M})$. 
 Let $\mathcal{A}$ be the set of all constructible arguments based on a $LeAr$ $( \mathcal{N}, \mathcal{K})$ and $\mathrm{LN}$. Define $\mathcal{M}$ as above. Let $\mathcal{\overline{K}}=\{\Rightarrow A|A\in \mathcal{K}\}$ and $\mathcal{W}=\mathcal{\overline{K}}\cup\mathcal{M}$, $Ag\in\mathcal{A}$ iff $\mathcal{W}\vdash_{\mathrm{LN}} \Rightarrow \conc(Ag)$. 
\end{remark}

\begin{definition}[Conflicts]\label{def:conflict}
	Let $Ag_1$, $Ag_2$, $Ag_2'\in \mathcal{A}$ be arguments. $Ag_1$ attacks $Ag_2$ on $Ag_2'\in\sub(Ag_2)$ if Conc($Ag_2'$)$=A$, and the following holds: 
    \begin{itemize}
    \item $Ag'_2=A\in\mathcal{K}$, or $Ag'_2$ is of the form $Ag''_1, \ldots, Ag''_n\Rightarrow_nA$;
    \item Conc($Ag_1$)$=\neg A$.
    \end{itemize}
\end{definition} 

For any $Q\subseteq\mathcal{K}\subseteq\mathcal{L}$, let $Cl_{\mathcal{R}_s}(Q)$ denote the closure of $Q$ under strict rules, i.e., all consequences of $Q$ in LN. Let $Q\vdash_{\mathrm{LN }}A$ denote that there exists an argument $Ag$ constructed by only strict rules, such that $\texttt{Prem}(Ag)\subseteq Q$ and $\texttt{Conc}(Ag)=A$.

\begin{definition}[Consistency]\label{def:consistency}
	$Q\subseteq\mathcal{L}$ is \textit{consistent} iff for any $A,B\in Cl_{\mathcal{R}_s}(Q)$, \(A\not=\neg B\);  \\
	
\end{definition}

\begin{remark}
Clearly, LN is a consistent logic. Therefore, $Q$ is consistent iff for any $A,B\in Cl_{\mathcal{R}_s}(Q)$, \(A\not=\neg B\).   
\end{remark}
\begin{remark}
We do not differentiate direct and indirect consistency as in \cite{MP13}. This is because we assume $\mathcal{R}_s$ contains all provable sequents in LN with $\mathcal{\overline{N}}$. Hence there is no difference between the direct and indirect consistency in this context.
\end{remark}

\begin{definition}[Preferences on arguments]\label{def:prefarg}
 For any $Ag_1$, $Ag_2\in\mathcal{A}$,
 $Ag_1\prec_x Ag_2$ if  $nx(\mathrm{Conc}(Ag_1))<nx(\mathrm{Conc}(Ag_2))$.
\end{definition}

In the following, without confusion, we denote arbitrary $<_i$ by $<$. 
\begin{definition}[Strict continuation]
For any set of arguments $\{Ag_1, \ldots, Ag_n\}$, $Ag$ is a \textit{strict continuation} of $\{Ag_1, \ldots, Ag_n\}$ if: 
\begin{itemize}
    \item $\norms(Ag)=\bigcup_{i=1}^n\norms(Ag_i)$,
    \item $\sr(Ag)\supseteq\bigcup_{i=1}^n\sr(Ag_i)$,
    \item $\prem(Ag)\supseteq\bigcup_{i=1}^n\prem(Ag_i)$.
\end{itemize}  
\end{definition}

The following definition of reasonable argument ordering for LeAr is modified from~\cite{MP13}.
\begin{definition}[Reasonable preferences for arguments]\label{def-reasonableOrdering}
	Let $\prec$ be a preference ordering on $\mathcal{A}$ constructed based on a LeAR.  $\prec$ is reasonable if $\forall Ag_1, Ag_2\in\mathcal{A}$,  let $Ag_1^+\in\mathcal{A}$ be a strict continuation of $\{Ag_1\}$; if $Ag_1\nprec Ag_2$, then $Ag_1^+\nprec Ag_2$, and if $Ag_2\nprec Ag_1$, then $Ag_2\nprec Ag_1^+$.
		
Let $ \preceq$ be defined naturally. Hence  $\preceq$ is reasonable if $\prec$ is reasonable.
\end{definition}

\begin{proposition}\label{pro-resonableOrder-last}
	The preference ordering $\prec$ ($\preceq$) on arguments is reasonable.
\end{proposition}

\begin{definition}[Argument evaluation]\label{def:AF}
	Let $\langle \mathcal{A}, \mathcal{D}_x\rangle$ be an AF, where $\mathcal{A}$ is the set of all the arguments constructed based on a LeAr, and $\mathcal{D}_x$ are the sets of defeats between arguments with respect to ordering $<_x$. 
	For all arguments $Ag_1$, $Ag_2\in\mathcal{A}$, $(Ag_1, Ag_2)\in\mathcal{D}_x$ if $Ag_1$ attacks $Ag_2$ and $Ag_1\nprec_x Ag_2$. A set of arguments  $E\subseteq\mathcal{A}$ is \textit{x-conflict-free} if  there are no $Ag_1, Ag_2\in E$ such that $(Ag_1, Ag_2)\in \mathcal{D}_x$. An argument $Ag_1$ is said to be \textit{x-defended} by $E$, if for any $Ag_2\in \mathcal{A}$, if $(Ag_2, Ag_1)\in \mathcal{D}_x $, then there is a $Ag_3\in {E}$ such that $(Ag_3, Ag_2)\in \mathcal{D}_x$.   A set of arguments  $E_x\subseteq\mathcal{A}$ is said to be \textit{x-complete}, if: 1) $E_x$ is  x-conflict-free, 2) for any  $Ag\in E_x$, $Ag$ is x-defended by $E_x$ and 3) for any $Ag\in\mathcal{A}$ x-defended by $E$, $Ag\in E_x$.
\end{definition}
In the following, without confusion, we denote arbitrary $\mathcal{D}_x$ by $\mathcal{D}$ and $E_x$ by $E$. 
\begin{proposition}\label{lem-ABdefeat}
	For any $Ag_1, Ag_2\in\mathcal{A}$, 
	\begin{enumerate}
		\item if $Ag_1'\in \texttt{Sub}(Ag_1)$ such that $(Ag_2, Ag_1')\in \mathcal{D}$, then $(Ag_2, Ag_1)\in \mathcal{D}$; 
		\item 	if $\mathrm{nnConc}(Ag_2)=0$ and $Ag_1$ attacks $Ag_2$, then $(Ag_1, Ag_2)\in\mathcal{D}$;
 		\item if $Ag_1$, $Ag_2$ attack each other, then one of the following cases holds: 
		 $(Ag_1, Ag_2)\in \mathcal{D}$; ii) $(Ag_2, Ag_1)\in \mathcal{D}$.
	\end{enumerate}
\end{proposition}

\begin{proof}
Follows directly from the definitions above.
\end{proof}

\begin{proposition}\label{lem-cA}
	For any complete extension $E=\{Ag_{1}$ , $\ldots$, $Ag_{n}\}$, if there exists an argument $Ag$ such that $Ag$ is a strict continuation of $E$, then $Ag\in E$.
\end{proposition}
\begin{proof}
Suppose that $Ag\not\in E$. Then by Definition \ref{def:AF}, $Ag$ is not defended by $E$. 
Therefore, there is an argument $Ag'\in \mathcal{A}$, such that $Ag'$ defeats $Ag$ on $Ag''\in\sub(Ag)$, and $\nexists Ag_x \in E$ such that $Ag_x$ attacks $Ag'$ or $Ag_x\prec Ag'$. 
We consider two possibilities. If $Ag''\not=Ag$, then by Definition \ref{def:conflict} and \ref{def:AF}, there exists $Ag_i$ ($1\leq i\leq n$) such that $(Ag',Ag_i)\in \mathcal{D}$,  contradicting the fact that $E$ is a complete extension defending $Ag_i$. Otherwise, $Ag'$ directly attacks $Ag$. 
Then there exists $\conc(Ag_1),\ldots,$ $\conc(Ag_k)\Rightarrow \conc(Ag)\in \norms(Ag)$ where $\conc(Ag_j)=\alpha_j$ for $1\leq j\leq k$. Hence  $\alpha_1,\ldots,\alpha_k\Rightarrow_n \conc(Ag_i)\in \norms(Ag_i)$ ($1\leq i\leq n$). 
Thus by Definition \ref{def:conflict} and \ref{def:AF}, $(Ag', Ag_i)\in \mathcal{D}$, again contradicting $E$ being a complete extension that defends $Ag_i$.
 
\end{proof}

By Proposition \ref{lem-ABdefeat}, \ref{lem-cA} one can easily obtain the following theorem. Similar proofs can be found in \cite{MP13}.

\begin{theorem}\label{pr:RP}
	Let $E$ be a complete extension based on LeAr. The following properties hold.
	\begin{itemize}
		\item $\forall Ag\in E$, if $Ag'\in \texttt{Sub}(Ag)$, then $Ag'\in E$. (Sub-argument Closure)
		\item $\{\texttt{Conc}(Ag|Ag\in E\}=Cl_{\mathcal{R}_s}(\{\texttt{Conc}(Ag)|Ag\in E\})$. (Closure under Strict Rules)
		\item $\{\texttt{Conc}(Ag)|Ag\in E\}$ is consistent. (Consistency)
	\end{itemize}
\end{theorem}
By theorem \ref{pr:RP} we show that LeAR based on LN satisfies the basic \textit{rationality postulates} according to Caminada and Amgoud \cite{CA07}.

\section{Application}\label{sec:application}

This section illustrates the system's function through a simple case:
\begin{example}\label{UK USA}
A UK-compliant autonomous vehicle is adapted for the US market, balancing costs, efficiency, and legal compliance across varied state regulations.
\end{example}

The logic system LN allows formulas to be labelled with natural numbers while supporting basic reasoning functions (both certain and uncertain). This enables us to introduce priorities in a concise manner using numerical representations as needed. 

Through preliminary research on traffic regulations in the UK and US states, we divide the priority of legal rules in design into four levels according to the different legal mandatory nature, demonstrated as follows: 
\begin{itemize}
    \item Mandatory and prohibitive requirements are labelled as 4.
    \item Desirable requirements are labelled as 3.
    \item Recommended requirements are labelled as 2. 
    \item Permissible requirements are labelled as 1.
\end{itemize}
 
We extracted portions of UK traffic laws and US traffic regulations in various states to form two sets of rules for Example \ref{UK USA}. The comparison of traffic rules presented in Table \ref{tab:seletedrules} illustrates the potential differences and the extent of variation between British and American traffic rules.

\begin{footnotesize}

\begin{table*}
    \centering
    \begin{tabular}{c|p{0.43\textwidth}|p{0.43\textwidth}}
\toprule
        &\textbf{Traffic rules from the UK} &  \textbf{Traffic rules from US}\\ \hline
         Rule 1: &Vehicles \textit{must} drive on the \textcolor{red}{left} side of the roadway unless otherwise directed.& 
         Vehicles \textit{must} drive on the \textcolor{red}{right} side of the roadway unless otherwise directed. 
         \\ 
         Rule 2: &Drivers \textit{must} give way to emergency vehicles when safe to do so.& 
         Drivers \textit{must} give way to emergency vehicles when safe to do so. (\textbf{same})\\ 
         Rule 3: &Vehicles \textit{must} display number plates \textcolor{red}{both at the front and rear}. & 
         Vehicles \textit{must} display a number plate \textcolor{red}{at the rear}./Front number plates are \textcolor{red}{not required} in some states.
         \\ 
         Rule 4: &Vehicle front hoods \textit{must} have energy-absorbing structures. &
          (\textbf{empty})
         \\ 
         Rule 5: &Vehicles are \textcolor{red}{\textit{prohibited} from making a turn} at a red light. &
         Vehicles are \textcolor{red}{\textit{allowed} to make a turn} in the direction of the lane (right or left) at a red light. 
         \\ 
         Rule 6: &It is \textcolor{red}{\textit{recommended}} to use the headlights during the day to improve visibility.&
         In some states, daytime running lights are \textcolor{red}{\textit{mandatory}}. \\  
         Rule 7: &Drivers \textit{should} maintain a safe distance from the vehicle in front \textcolor{red}{by following the general rule of a two-second gap}. 
          &Drivers \textit{should} maintain a safe following distance based on speed and conditions. (\textbf{similar})
          \\ 
          Rule 8: &In the event of a minor accident causing only slight damage to personal property, it is \textcolor{red}{\textit{permissible} for drivers to handle the matter with their insurance companies}. &
         If property damage exceeds a set amount (e.g., \$500), \textcolor{red}{reporting to the police is \textit{mandatory}}. 
                   \\
                      \bottomrule
    \end{tabular}
    \caption{Comparison of US and UK traffic rules}
    \label{tab:seletedrules}
\end{table*}
\end{footnotesize}

Following the labelling principles outlined above, we assign numerical labels to each formula representing legal provisions. For instance, Rule 1 in Table \ref{tab:seletedrules} yields the following formalizations: $DriveLeft(AV)\cdot1\cdot 4\Rightarrow_n\neg DriveRight(AV)\cdot1\cdot4 $ and $DriveRight(AV)\cdot2\cdot4\Rightarrow_n\neg DriveLeft(AV)\cdot2\cdot4$, where the first numerical label following a formula distinguishes the country (1 for the initial country and 2 for the target country), and the second label indicates the strength of the legal modality.

According to rules in Table \ref{tab:seletedrules} and the logical system introduced in Section~\ref{sec:formal}, we can construct the following argumentation theory. 

\begin{footnotesize}

\begin{equation*}
\mathcal{K}
    =\left\{
    \begin{array}{l}
  DriveLeft(AV) \cdot1 \cdot4, DriveRight(AV)\cdot2\cdot4, 
  Emergency(V), \\NumPlate(AV, Rear)\cdot1\cdot4, NumPlate(AV, Front)\cdot1\cdot4, \\
         NumPlate(AV, Rear)\cdot2\cdot4, NumPlate(AV, Front)\cdot2\cdot1, \\EnergyAbsorbing(AV, FrontHood)\cdot1\cdot4, \\
         \neg Turn(AV, RedLight)\cdot1\cdot4, 
         Turn(AV, RedLight)\cdot2\cdot1, \\PoorVision(Day), Front(V),\\
         SlightDamage(AV), Damage>\$500(AV)
    
    \end{array}
 \right\}
\end{equation*}

\begin{equation*}
\mathcal{N}
    =\left\{
\begin{array}{ll}
     DriveLeft(AV)\cdot1\cdot 4\Rightarrow_n\neg DriveRight(AV)\cdot1\cdot4 \\
     DriveRight(AV)\cdot2\cdot4\Rightarrow_n\neg DriveLeft(AV)\cdot2\cdot4\\
    Emergency(V)\Rightarrow_n GiveWay(AV, V)\cdot1\cdot4\\
         Emergency(V)\Rightarrow_n GiveWay(AV, V)\cdot2\cdot4\\
         PoorVision(Day)\Rightarrow_n TurnOnHeadlights(AV, Day)\cdot1\cdot2\\
         PoorVision(Day)\Rightarrow_n TurnOnHeadlights(AV, Day)\cdot2\cdot4\\ 
         Front(V)\Rightarrow_n TwoSecondGap(AV, V)\cdot1\cdot3\\
         Front(V)\Rightarrow_n SafeDistance(AV, V)\cdot2\cdot3\\
          SlightDamage(AV)\Rightarrow_n Insurance(AV)\cdot1\cdot1\\ Damage>\$500(AV)\Rightarrow_n CallPolice(AV)\cdot2\cdot4\\
\end{array}
 \right\}
\end{equation*}
\end{footnotesize}

Based on Definition \ref{def:argument} and \ref{def:conflict},  we can construct and identify the following conflicting arguments.
\footnote{To make the distinction easier, we use $B$ to denote arguments constructed according to the UK rules and $A$ for those constructed according to the US rules.} 
The conflict relation between these arguments is shown in Figure~\ref{fig:conflict_1}.

\begin{footnotesize}
  \begin{tabular*}{0.4\textwidth}{@{\extracolsep{\fill} }ll}
     $B_1: DriveLeft(AV)\cdot1\cdot 4$ \\ $A_1: DriveRight(AV)\cdot2\cdot 4\Rightarrow_n \neg DriveLeft(AV)\cdot2\cdot 4$\\
 $B_2: DriveLeft(AV)\cdot1\cdot 4\Rightarrow_n \neg DriveRight(AV)\cdot1\cdot 4$\\
		$A_2: DriveRight(AV)\cdot2\cdot 4$\\
		$B_3: \neg Turn(AV, RedLight)\cdot1\cdot4$\\
		$A_3: Turn(AV, RedLight)\cdot2\cdot1$\\
		
            \end{tabular*}
\end{footnotesize}

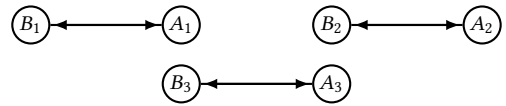
\begin{figure}[h]
	\[
	\begin{tabular}{c}
		\begin{tikzpicture}[->,>=latex,thick,shorten >=1pt,font=\small,scale=0.5] 
			\vertex (B1) at (0,3) {$B_{1}$};
			\vertex (A1) at (4,3) {$A_{1}$};
            \vertex (B2) at (8,3) {$B_{2}$};
			\vertex (A2) at (12,3) {$A_{2}$};
			\draw [->] (B1) -- (A1);
            \draw [->] (A1) -- (B1);
            \draw [->] (B2) -- (A2);
            \draw [->] (A2) -- (B2);
		\end{tikzpicture}
	\end{tabular}
	\]
    \[
	\begin{tabular}{c}
		\begin{tikzpicture}[->,>=latex,thick,shorten >=1pt,font=\small,scale=0.5] 
			\vertex (B3) at (0,3) {$B_{3}$};
			\vertex (A3) at (4,3) {$A_{3}$};
			\draw [->] (A3) -- (B3);
            \draw [->] (B3) -- (A3);
		\end{tikzpicture}\\
	\end{tabular}
	\]
	\vspace*{-5pt}\caption{Conflicts between arguments}\label{fig:conflict_1}
\end{figure}

If we adjust based on the importance of regulations, prioritization and the final conflict relations between arguments should be determined by the second numerical label to avoid more severe penalties. Therefore, the attack from $A_3$ to $B_3$ is deleted since $1<4$. 

When regulations are of equal importance, users may prefer to match the regulations of the target country. In this case, based on the first numerical label on a formula, regulations marked as 2 take precedence over those marked as 1 (i.e., if the US is the target country, its regulations override those of the UK). This leads to the elimination of attacks from $B_1$ to $A_1$ and from $B_2$ to $A_2$. 
Therefore, within the set of jointly acceptable conclusions derived from the attack relations among arguments $A_1$, $A_2$, $B_1$, and $B_2$ (cf. Definition \ref{def:AF}), it is reasonable to select the set containing conclusions $DriveRight(AV)$ and $\neg DriveLeft(AV)$.

Based on the formal theory proposed in Section \ref{sec:formal}, at least the following decision principles can be supported for argument evaluation and conclusion selection:

\begin{description}
    \item[Minimal Adjustment] Adopt higher-priority regulations while preserving non-conflicting rules from the original country.
    \item[Maximum Consistency] Fully align with the jurisdiction holding the highest priority.
    \item[Caution First] Maintain all non-conflicting legal requirements from both jurisdictions.
\end{description}

Each principle addresses distinct requirements: 
\textit{Minimal Adjustment} optimizes efficiency; 
\textit{Maximum Consistency} enforces complete compliance; 
\textit{Caution First} maximizes legal safety through conflict avoidance. 
Additionally, numerical labels have the potential to support further strategies like \textit{Gradual Transition}, which allows for gradual regulatory adaptation with minimal changes at each stage.

\section{Conclusion}\label{sec:conc}

This paper proposes the use of formal argumentation as a tool for non-monotonic reasoning, incorporating Lambek calculus with negation as the foundational logic. The basic aim is to deliver an inference system that is both easy to understand and use, while supporting a robust knowledge base and defeasible knowledge representation, specifically within the context of cross-border autonomous driving applications. The introduction of the logical foundation allows the system to derive the necessary strict reasoning rules and ensure that it meets rationality requirements. Meanwhile, the defeasible knowledge representation empowers users to flexibly express uncertain inference relationships.

The weights represented by labels in the form of natural numbers offer intuitive comprehension and facilitate easy tracing of the origins of priority levels. Taken together, this system has the potential to serve as a user-friendly non-monotonic reasoning tool for professionals in the fields of law and autonomous vehicle design, as well as for users of autonomous vehicles.








\bibliographystyle{ACM-Reference-Format}
\bibliography{logiclaw}










\end{document}